\documentclass{article}

\usepackage[nonatbib, preprint]{nips_2018}

\usepackage{times}
\usepackage{cite}
\usepackage{subfigure}
\usepackage{graphicx}
\usepackage{amsmath}
\usepackage{bm}
\usepackage{url}
\usepackage{stmaryrd}
\usepackage{balance}
\usepackage{amssymb}
\usepackage{pgfplots}
\usepackage{pifont}
\usepackage[usenames,dvipsnames]{pstricks}
\usepackage{epsfig}
\usepackage{booktabs}
\usepackage{pgffor}
\usepackage{tikz}
\usepackage{multirow}

\newcommand{\mb}{\mathbf}

\newtheorem{theo}{\textsc{Theorem}}
\newtheorem{lemma}{\textsc{Lemma}}
\newtheorem{proof}{\textsc{Proof}}

\usepackage{algorithm}
\usepackage{algorithmic}

\title{Reconciled Polynomial Machine: A Unified Representation of Shallow and Deep Learning Models}

\author{Jiawei~Zhang$^\star$, Limeng Cui$^\dagger$, Fisher B. Gouza$^\dagger$\\
$^\star$IFM Lab, Florida State University, FL, USA\\
$^\dagger$University of Chinese Academy of Sciences, Beijing, China\\
jzhang@cs.fsu.edu, lmcui932@163.com, fisherbgouza@gmail.com}

\begin{document}

\maketitle


\vspace{-20pt}
\begin{abstract}
\vspace{-10pt}
In this paper, we aim at introducing a new machine learning model, namely \textit{reconciled polynomial machine}, which can provide a unified representation of existing shallow and deep machine learning models. \textit{Reconciled polynomial machine} predicts the output by computing the inner product of the \textit{feature kernel function} and \textit{variable reconciling function}. Analysis of several concrete models, including Linear Models, FM, MVM, Perceptron, MLP and Deep Neural Networks, will be provided in this paper, which can all be reduced to the \textit{reconciled polynomial machine} representations. Detailed analysis of the learning error by these models will also be illustrated in this paper based on their reduced representations from the function approximation perspective.

\end{abstract}
\vspace{-10pt}
\vspace{-10pt}
\section{Introduction}\label{sec:intro}
\vspace{-10pt}

Generally, the conventional machine learning problems aim at recovering a mathematical mapping from the feature space to the label space. We can represent the unknown true mapping steering the real-world data distribution as $g (\cdot; \mb{\theta}): \mathcal{X} \to \mathcal{Y}$, where $\mathcal{X}$ and $\mathcal{Y}$ denote the feature space and label space respectively, and $\mb{\theta}$ represents the variables in the mapping function. Depending on the application settings, such a mapping can be either a simple or a quite complicated equation involving both the variables and extracted features. 

To approximate such a mapping, various machine learning models have been proposed which can be trained based on a small amount of feature-vector pairs $\mathcal{T} = \{(\mb{x}_i, \mb{y}_i)\}_{i = 1}^n$ sampled from space $\mathcal{X} \times \mathcal{Y}$, where $\mb{y}_i = g (\mb{x}_i; \mb{\theta})$. Formally, we can represent the approximated mapping outlined by the general machine learning models (including deep neural networks) from the feature space to the label space as $f(\cdot; \mb{w}): \mathcal{X} \to \mathcal{Y}$ parameterized by $\mb{w}$. By minimizing the following objective function, the machine learning models can learn the optimal variable $\mb{w}^*$:
\begin{equation}
\mb{w}^* = \min_{\mb{w} \in \mathcal{W}} \mathcal{L} \left( f(\mb{X}; \mb{w}), \mb{y} \right),
\end{equation}
where $\mathcal{L}(\cdot, \cdot)$ represents the loss function of the prediction results compared against the ground truth, and $\mathcal{W}$ denotes the feasible variable space. Terms $\mb{X} = [\mb{x}_1^\top, \mb{x}_2^\top, \cdots, \mb{x}_n^\top]^\top$ and $\mb{y} = [\mb{y}_1^\top, \mb{y}_2^\top, \cdots, \mb{y}_n^\top]^\top$ represent the feature matrix and label vector of the training data respectively.

Existing machine learning research works mostly assume mapping $f(\cdot; \mb{w}^*)$ with the optimal variables can provide a good approximation of the true mapping $g (\cdot; \mb{\theta})$ between the feature and label space $\mathcal{X}$ and $\mathcal{Y}$. Here, variable $\mb{w}^*$ can be either locally optimal or globally optimal depending on whether the loss function $\mathcal{L} \left( f(\mb{X}; \mb{w}), \mb{y} \right)$ is convex or not regarding variable $\mb{w}$. Meanwhile, according to the specific machine learning algorithm adopted, function $f(\cdot; \mb{w})$ usually has very different representations, e.g., weighted linear summation for linear models, probabilistic projection for graphical models, and nested projections for deep neural networks via various non-polynomial activation functions.

In this paper, we will analyze the errors introduced by the learning model $f(\cdot; \mb{w}^*)$ in approximating the true mapping function $g (\cdot; \mb{\theta})$. Literally, we say function $g (\cdot; \mb{\theta})$ can be \textit{approximated} with model mapping $f(\cdot; \mb{w}^*)$ based on a dataset $\mathcal{T}$ iff (1) function $f(\cdot; \mb{w}^*)$ can achieve low \textit{empirical  loss} on the training dataset $\mathcal{T}$, and (2) the \textit{empirical  loss} should also be close to the \textit{expected loss} of $f(\cdot; \mb{w}^*)$ on space $\mathcal{X} \times \mathcal{Y}$ as well. To achieve such objectives, based on different prior assumptions about the distribution of the learning space $\mathcal{X} \times \mathcal{Y}$, the existing machine learning model mapping functions in different forms can usually achieve very different performance. In this paper, we try to provide a unified representation of the diverse existing machine learning algorithms, and illustrate the reason why they can obtain different learning performance.


\vspace{-10pt}
\section{Unified Machine Learning Model Representation}\label{sec:representation}
\vspace{-10pt}


%
%

In the following part of this paper, to simplify the analysis, we will assume the labels of the data instances are real numbers of dimension $1$, e.g., $y_i \in \mathbb{R}$, and the feature vectors of data instances are binary of dimension $d_1$, e.g., $\mb{x}_i \in \{0, 1\}^{d_1}$ (i.e., $\mathcal{X} = \{0, 1\}^{d_1}$ and $\mathcal{Y} = \mathbb{R}$). For the label vectors of higher dimensions and continuous feature values, similar analysis results will be also achieved.

\vspace{-5pt}

\begin{theo}\label{theo:representation}
Based on the space $\mathcal{X} \times \mathcal{Y}$, for any functions $g(\mb{x} ; \mb{\theta})$, where $\mb{\theta} \in \mathbb{R}^{d_2}$ and $\mb{x} \in \{0, 1\}^{d_1}$, they can all be represented as a finite weighted sum of monomial terms about $\mb{x}$ as follows:\vspace{-5pt}\begingroup\makeatletter\def\f@size{8}\check@mathfonts
\begin{equation}
g(\mb{x} ; \mb{\theta}) = \sum_{n = 0}^{d_1} \sum_{1 \le i_1 < i_2 < \cdots < i_n \le d_1} v^{(n)}_{{i_1,i_2, \cdots, i_n}} \cdot  x_{i_1} x_{i_2} \cdots x_{i_n} = \psi(\mb{\theta})^\top \kappa(\mb{x}),
\end{equation} \endgroup

\vspace{-10pt}
\noindent where term $v^{(n)}_{{i_1,i_2, \cdots, i_n}} \in \mathbb{R}$ denotes a weight computed based on $\mb{\theta}$. Functions $\kappa(\cdot): \{0, 1\}^{d_1} \to \mathbb{R}^D$ and $\psi(\cdot): \mathbb{R}^{d_2} \to \mathbb{R}^D$ project the feature and variable vectors to a space of the same dimension, which are called the ``\textit{feature kernel function}'' and ``\textit{variable reconciling function}'' respectively.
\end{theo}

\vspace{-5pt}

Here, we need to add a remark, the original feature and variable vectors can be of different dimensions actually, i.e., $d_1 \neq d_2$. There exist various definitions of functions $\kappa(\cdot)$ and $\psi(\cdot)$. For instance, according to the above equation, we can provide an example of these functions as follows, which projects the features and variables to a shared space of dimension $D = 2^{d_1}$: \vspace{-5pt} \begingroup\makeatletter\def\f@size{8}\check@mathfonts
\begin{align}
\psi(\mb{\theta})&= [v^{(0)}; \ v^{(1)}_1, \cdots, v^{(1)}_{d_1}; \ v^{(2)}_{1,2}, \ \ \ \ \ \cdots, v^{(2)}_{d_1-1,d_1}; \ \ \ \ \ v^{(3)}_{1,2,3}, \ \ \ \ \ \ \ \cdots; \ v^{(N)}_{1,2,\cdots,d_1}]^\top,\\
\kappa(\mb{x}) &= [1; \ \ \ \ \ x_1, \ \ \ \cdots, x_{d_1}; \ (x_1x_2), \cdots, (x_{d_1-1}x_{d_1}); \ (x_1x_2 x_3), \cdots; \ (x_1x_2 x_{d_1})]^\top.
\end{align}\endgroup

\vspace{-12pt}

\begin{proof}
Formally, given a mapping function $g(; \mb{\theta})$, the true label $y \in \mathbb{R}$ of data instance featured by vector $\mb{x} = [x_1, x_2, \cdots, x_{d_1}]^\top \in \mathbb{R}^d$ can be represented as $y = g(\mb{x}; \mb{\theta})$. At $\mb{x} = \mb{x}_0$, the mapping function can be represented as the weighted summation of polynomial terms according to the Taylor expansion theorem, i.e.,\vspace{-8pt} \begingroup\makeatletter\def\f@size{8}\check@mathfonts
\begin{align}
g(\mb{x}; \mb{\theta}) &= g(\mb{x}_0; \mb{\theta}) + \hspace{-5pt} \sum_{i_1 = 1}^{d_1} \frac{\partial g(\mb{x}_0; \mb{\theta})}{\partial x_{i_1}} \cdot (x_{i_1} - x_{0,i_1}) + \hspace{-5pt} \sum_{i_1, i_2 = 1}^{d_1} \frac{\partial^2 g(\mb{x}_0; \mb{\theta})}{\partial x_{i_1} \partial x_{i_2}} \cdot (x_{i_1} - x_{0,i_1}) (x_{i_2} - x_{0,i_2}) + \cdots ,\\
&= \sum_{n = 1}^{N} \sum_{i_1, \cdots, i_n = 0}^{d_1} \frac{\partial^n g(\mb{x}_0; \mb{\theta}) }{\partial x_{i_1} \cdots \partial x_{i_n}} (x_{i_1} - x_{0,i_1}) (x_{i_2} - x_{0,i_2}) \cdots (x_{i_n} - x_{0,i_n}),\\
&\overset{\underset{\mathrm{\mb{x}_0 = \mb{0} }}{}}{=} \sum_{n = 0}^{N}  \sum_{i_1, \cdots, i_n = 1}^{d_1} \frac{\partial^n g(\mb{x}_0; \mb{\theta}) }{\partial x_{i_1} \cdots \partial x_{i_n}} x_{i_1} x_{i_2} \cdots x_{i_n}.
\end{align}\endgroup
Here, the key point is whether $N$ is finite or $N$ will be approaching $\infty$. Depending on the highest order of polynomial terms involved in the true mapping $g(\mb{x}_0; \mb{\theta})$, we have the following two cases:
\begin{itemize}
\vspace{-5pt}
\item \textit{case 1}: In the case when the largest order of polynomial term in $g(\mb{x}_0; \mb{\theta})$ is a finite number $k$, it is easy to show that $\frac{\partial^{j} g(\mb{x}_0; \mb{\theta}) }{\partial x_1 \cdots \partial x_j} = 0, \forall j > k$. In other words, we have $N$ to be a finite number as well and $N=k$.
\vspace{-3pt}
\item \textit{case 2}: In the case when derivative of function $g(\mb{x}_0; \mb{\theta})$ exists for any $j \in [1, \infty]$. It seems the equation will have an infinity number of polynomial terms. Meanwhile, considering that $x_i \in \{0, 1\}$, we have the $j_{th}$ power of $x_i$ will be equal to $x_i$, i.e., $x_i^j = x_i$ for any $j \in \{1, 2, \cdots, \infty\}$. In other words, the high-order polynomial term $x_{i_1}^j x_{i_2}^k \cdots x_{i_n}^m = x_{i_1} x_{i_2} \cdots x_{i_n} $ can always hold, which will reduce the infinity number of polynomial terms to finite polynomial terms of $x_1, x_2, \cdots x_{d_1}$ instead, i.e., $N$ is still a finite number.
\vspace{-5pt}
\end{itemize}

Based on the above analysis, we can simplify the above equation as follows  \vspace{-5pt} \begingroup\makeatletter\def\f@size{8}\check@mathfonts
\begin{equation}
\sum_{n = 0}^{N}  \sum_{i_1, \cdots, i_n = 1}^{d_1} \frac{\partial^n g(\mb{x}_0; \mb{\theta}) }{\partial x_{i_1} \cdots \partial x_{i_n}} x_{i_1} x_{i_2} \cdots x_{i_n} = \sum_{n = 0}^{d_1} \sum_{1 \le i_1 < i_2 < \cdots < i_n \le d_1} v^{(n)}_{{i_1,i_2, \cdots, i_n}} \cdot  x_{i_1} x_{i_2} \cdots x_{i_n},
\end{equation}\endgroup

\vspace{-10pt}
\noindent where the weight terms $v^{(n)}_{{i_1,i_2, \cdots, i_n}}$ is the sum of both the $n_{th}$-order derivative value $\frac{\partial^n g(\mb{x}_0; \mb{\theta}) }{\partial x_{i_1} \cdots \partial x_{i_n}}$ as well as even higher order of derivatives, e.g., $\frac{\partial^{n+1} g(\mb{x}_0; \mb{\theta}) }{\partial x_{i_1} \partial x_{i_1} \cdots \partial x_{i_n}}$ and so forth.
\end{proof}

\vspace{-10pt}
\subsection{Approximation Error Analysis}
\vspace{-8pt}

It is similar for the mappings $f(\mb{x}; \mb{w})$ of the machine learning model, which can also be represented as a polynomial summation, i.e., $f(\mb{x}; \mb{w}) = \psi'(\mb{w})^\top \kappa'(\mb{x})$, where $\mb{w} \in \mathbb{R}^{d_3}$, $\psi'(\cdot): \mathbb{R}^{d_3} \to \mathbb{R}^{D'}$ and $\kappa'(\cdot): \{0, 1\}^{d_1} \to \mathbb{R}^{D'}$. Given the learning space $\mathcal{X} \times \mathcal{Y}$, the approximation process of function $f(\mb{x}; \mb{w})$ for true mapping $g(\mb{x}; \mb{\theta})$ generally involves $4$ key factors:
\begin{enumerate}
\vspace{-5pt}
\item dimension of parameter $\mb{w}$: ${d_3}$, 
\vspace{-3pt}
\item objective learning space dimension: $D'$,
\vspace{-3pt}
\item weight reconciling function $\psi'(\cdot): \mathbb{R}^{d_3} \to \mathbb{R}^{D'}$,
\vspace{-3pt}
\item feature kernel function $\kappa'(\cdot): \{0, 1\}^{d_1} \to \mathbb{R}^{D'}$. 
\vspace{-5pt}
\end{enumerate}

If $f(\mb{x}; \mb{w})$ can pick the identical factors as function $g(\mb{x}; \mb{\theta})$, literally $f(\mb{x}; \mb{w})$ will precisely recover $g(\mb{x}; \mb{\theta})$. However, in real applications, precise recovery of $g(\mb{x}; \mb{\theta})$ is usually an impossible task. Meanwhile, according to the Vapnik-Chervonenkis theory \cite{VC15, BEHW89}, for measuring the quality of function $f(\mb{x}; \mb{w})$, we can compute the error introduced by it compared against the true mapping function $g(\mb{x}; \mb{\theta})$ based on the learning space $\mathcal{X}$, which can be represented as\begingroup\makeatletter\def\f@size{8}\check@mathfonts
\begin{equation}
 \begin{matrix} 
 \underbrace{ \int_{\mb{x} \in \mathcal{X}} \hspace{-10pt} p(\mb{x}) \left\| f(\mb{x}; \mb{w}) - g(\mb{x}; \mb{\theta}) \right\| \mathrm{d}\mb{x} } = \hspace{-5pt}
 & \underbrace{ \int_{\mb{x} \in \mathcal{T}} \hspace{-10pt} p(\mb{x}) \left\| f(\mb{x}; \mb{w}) - g(\mb{x}; \mb{\theta}) \right\| \mathrm{d}\mb{x}  }  + \hspace{-5pt}
 &\underbrace{ \int_{\mb{x} \in \mathcal{X} \setminus \mathcal{T} } \hspace{-10pt} p(\mb{x}) \left\| f(\mb{x}; \mb{w}) - g(\mb{x}; \mb{\theta}) \right\| \mathrm{d}\mb{x}  },\\
 \mbox{overall loss } \mathcal{L} & \mbox{empirical loss } \mathcal{L}_{em} &\mbox{expected loss } \mathcal{L}_{ex}
 \end{matrix}
\end{equation}\endgroup
where $p(\mb{x})$ denotes the probability density function of $\mb{x}$ and $\left \| \cdot \right\|$ denotes a norm measuring the difference between $f(\mb{x}; \mb{w})$ and $g(\mb{x}; \mb{\theta})$.

Minimization of term $\mathcal{L}$ is equivalent to the minimization of $\mathcal{L}_{em}$ and $\mathcal{L}_{ex}$ simultaneously. In the learning process, the training data $\mathcal{T}$ is given but we have no idea about the remaining data instances $\mathcal{X} \setminus \mathcal{T}$. In other words, computation of the \textit{expected loss} term $\mathcal{L}_{ex}$ is impossible. Existing machine learning algorithms solve the problem with two-fold: (1) minimization of the \textit{empirical loss} $\mathcal{L}_{em}$, and (2) minimization of the gap between \textit{empirical loss} and \textit{overall loss}, i.e., $| \mathcal{L}_{em} - \mathcal{L} |$. To achieve such an objective, various different machine learning models have been proposed already. In the following section, we will illustrate how the existing machine learning algorithms determine the $4$ factors so as to minimize the model approximation loss terms.

\vspace{-10pt}
\section{Classic Machine Learning Model Approximation Analysis}\label{sec:analysis}
\vspace{-10pt}

In this section, we will provide a comprehensive analysis about the existing machine learning algorithms, and illustrate that they can all be represented as the inner product of the kernel function of features and the reconciling function about variables.

\vspace{-10pt}
\subsection{Linear Model Approximation Error Analysis}
\vspace{-8pt}

At the beginning , we propose to give an analysis of the linear models \cite{CV95, YS09} first, which will provide the foundations for studying more complicated learning models. Formally, given a data instance featured by vector $\mb{x} = [x_1, x_2, \cdots, x_{d_1}]^\top$ of dimension ${d_1}$, based on a linear model parameterized by the optimal weight vector $\mb{w}^* = [w_0, w_1, \cdots, w_{d_3 - 1}] \in \mathbb{R}^{d_3}$, we can represent the mapping result of the data instance as \vspace{-10pt} \begingroup\makeatletter\def\f@size{8}\check@mathfonts
\begin{equation}
\hat{y} = f(\mb{x}; \mb{w}) = w_0 + \sum_{i = 1}^{d_1} w_i \cdot x_i = \psi(\mb{w})^\top \kappa(\mb{x}).
\end{equation}\endgroup

\vspace{-10pt}
\noindent According to the representation, we have the $4$ factors for linear models as: (1) $d_3 = d_1 + 1$, (2) $D' = d_1 + 1$, (3) $\psi(\mb{w}) = \mb{w}$, and (4) $\kappa(\mb{x}) = [1, x_1, x_2, \cdots, x_{d_1}]^\top$. Compared with the true values, we can represent the approximation error by the model for $\mb{x}$ as\vspace{-10pt} \begingroup\makeatletter\def\f@size{8}\check@mathfonts
\begin{align}
\left\| \hat{y} - y \right\| &= \left\| f(\mb{x}; \mb{w}) - g(\mb{x}; \mb{\theta}) \right\|  = \left\| \left( w_0 + \sum_{i = 1}^{d_1} w_i \cdot x_i \right) - \left( g(\mb{0}; \mb{\theta}) + \sum_{i = 1}^{d_1} \frac{\partial g(\mb{0}; \mb{\theta}) }{\partial x_i} x_i + R_2(\mb{x}) \right) \right\|  \\
&= \left\| a + \mb{b}^\top \mb{x}  - R_2(\mb{x}) \right\| .
\end{align}\endgroup

\vspace{-12pt}
\noindent where $a = \big ( w_0 - g(\mb{0}; \mb{\theta}) \big)$ and $\mb{b} = [b_1, b_2, \cdots, b_{d_1}]^\top$ with entry $b_i =  w_i - \frac{\partial g(\mb{0}; \mb{\theta}) }{\partial x_i}$.

Literally, for the linear models, the approximation error is mainly introduced by approximating the high-order remainder term $R_2(\mb{x})$ with the linear function $b + \mb{a}^\top \mb{x}$. In other words, for the linear models, the empirical error term $\mathcal{L}_{em}$ is usually of a large value when dealing with non-linearly separable data instances. Even if $\mathcal{L}_{em}$ can provide a good approximation of the whole error term $\mathcal{L}$, the overall approximation performance will still be seriously bad in such situations. 

\vspace{-10pt}
\subsection{Quadratic Model Approximation Error Analysis}
\vspace{-8pt}

To resolve such a problem, in recent years, some research works propose to incorporate the interactions among the features into the model learning process, and several learning models, like FM (Factorization Machine) \cite{R10}, have been proposed. 

FM proposes to combine the advantages of linear models, e.g., SVM, with the factorization models. Formally, given the data instance featured by vector $\mb{x} = [x_1, x_2, \cdots, x_{d_1}]^\top \in \mathbb{R}^{d_1}$, the prediction label by FM can be formally represented as\vspace{-8pt} \begingroup\makeatletter\def\f@size{8}\check@mathfonts
\begin{equation}
\hat{y} = f(\mb{x}; \mb{w}) = w_0 + \sum_{i = 1}^{d_1} w_i \cdot x_i + \sum_{i=1}^{d_1} \sum_{j = i+1}^{d_1} w_{i,j} \cdot x_i x_j = \psi(\mb{w})^\top \kappa(\mb{x}).
\end{equation}\endgroup

\vspace{-10pt}
\noindent where $\psi(\mb{w}) = [w_0] \sqcup [w_i]_{i = 1}^{d_1} \sqcup [w_{i,j}]_{i, j = 1, i < j}^{d_1}$ denotes the variable vector. Operator $\sqcup$ denotes the concatenation of vectors.

For the data instances featured by vectors of dimension $d_1$, the total number of variables involved in FM will be $1+ d_1 + \frac{d_1 (d_1-1)}{2}$, learning of which a challenging problem for large $d_1$. To resolve such a problem, besides the weights $[w_0] \sqcup [w_i]_{i = 1}^{d_1}$ for linear and bias terms, FM introduces an extra factor vector to define weights in $[w_{i,j}]_{i, j = 1, i < j}^{d_1}$, which can be represented by matrix $\mb{V} \in \mathbb{R}^{d_1 \times k}$. Formally, FM defines the quadratic polynomial term weight as $w_{i,j} = \left \langle \mb{V}[i,:], \mb{V}[j,:] \right \rangle$, where $\mb{V}[i,:]$ and $\mb{V}[j,:]$ are the factor vectors corresponding to the $i_{th}$ and $j_{th}$ feature respectively.

Therefore, for the FM model, we have the $4$ key factors: (1) $d_3 = d_1+1 + d_1 \times k$, (2) $D' = d_1 + 1+\frac{d_1(d_1 - 1)}{2}$, (3) $\psi([w_0] \sqcup [w_i]_{i = 1}^{d_1} \sqcup vec(\mb{V}) ) = [w_0] \sqcup [w_i]_{i = 1}^{d_1} \sqcup [w_{i,j}]_{i, j = 1, i < j}^{d_1}$, and (4) $\kappa(\mb{x}) = [1] \sqcup [x_i]_{i = 1}^{d_1} \sqcup [x_i x_j ]_{i,j = 1, i < j}^{d_1}$. Here, $[w_0] \sqcup [w_i]_{i = 1}^{d_1} \sqcup vec(\mb{V})$ are the variables to be learned in FM.

\vspace{-10pt}
\subsection{Higher-Order Model Approximation Error Analysis}
\vspace{-8pt}

Meanwhile, the recent Multi-View Machine (MVM) \cite{CZLY16} proposes to partition the feature vector into several segments (each segment denotes a view), and consider even higher-order feature interactions among these views into modeling. Formally, we can represent the multi-view feature vector as $\mb{x} = \left[ (\mb{x}^{(1)})^\top, (\mb{x}^{(1)})^\top, (\mb{x}^{(2)})^\top, \cdots, (\mb{x}^{(m)})^\top \right]^\top$, where the superscript denotes the view index and $m$ is the total view number. The prediction result by MVM can be represented as \vspace{-5pt}\begingroup\makeatletter\def\f@size{8}\check@mathfonts
\begin{align}
\hat{y} &=  w_0 \hspace{-2pt} + \hspace{-3pt} \sum_{p=1}^m \sum_{i_p = 1}^{I_p} w_{i_p}^{(p)} x_{i_p}^{(p)} \hspace{-2pt} + \hspace{-3pt} \sum_{p=1}^m \sum_{q = p+1}^m \sum_{i_p = 1}^{I_p} \sum_{i_q = 1}^{I_q} \hspace{-2pt} w_{i_p,i_q}^{(p,q)}  x_{i_p}^{(p)} x_{i_q}^{(q)} \hspace{-3pt} + \cdots +  \hspace{-3pt} \sum_{i_1=1}^{I_1} \hspace{-3pt} \cdots \hspace{-5pt} \sum_{i_m=1}^{I_m} \hspace{-3pt} w_{i1,i2, \cdots, i_m}^{(1, 2, \cdots, m)} \hspace{-2pt} \left( \prod_{p=1}^m x_{i_p}^{(p)} \right) \\
&= \psi(\mb{w})^\top \kappa(\mb{x}).
\end{align}\endgroup
where $I_p$ denotes the feature length of the $p_{th}$ view, i.e., the length of vector $\mb{x}^{(p)}$.

For the higher-order variable, e.g., $w_{i_p,i_q}^{(p,q)}$, MVM also introduces a factorization style method to define the \textit{variable reconciling function} based on a sequence of matrices $\mb{A}^{(p)} \in \mathbb{R}^{(I_p \times k)}$ for the $p_{th}$ view, where $w_{i_1, i_2, \cdots, i_n} = \sum_{j = 1}^k \prod_{p = 1}^n A^{(p)}_{i_p, j}$. Therefore, the key $4$ factors involved in the MVM are as follows: (1) $d_3 = k \sum_{p=1}^m (I_p + 1) = k (d_1 + m)$, (2) $D' = \prod_{i = 1}^m (I_i + 1)$, (3) $\psi(vec(\mb{A}^{(1)}) \sqcup \cdots \sqcup vec(\mb{A}^{(m)}) ) = [w_0] \sqcup [w_{i_p}^{(p)}]_{p=1, i_p = 1}^{m, I_p} \sqcup \cdots \sqcup [w_{ i1,i2, \cdots, i_m }^{ (1, 2, \cdots, m) }]_{i_1, \cdots, i_m = 1}^{I_1, I_2, \cdots, I_m}$, and (4) $\kappa(\mb{x}) = [1] \sqcup [x_{i_p}^{(p)}]_{p=1, i_p = 1}^{m, I_p} \sqcup \cdots \sqcup [\prod_{p=1}^m x_{i_p}^{(p)}]_{i_1, \cdots, i_m = 1}^{I_1, I_2, \cdots, I_m}$.

The FM can be viewed as a special case of the MVM, which involves $d_1$ views denoted by each feature in the vector $\mb{x}$. Furthermore, compared against the output of true models, the error introduced by the MVM (with order $m$) on instance $\mb{x}$ can be represented as\begingroup\makeatletter\def\f@size{8}\check@mathfonts
\begin{align}
\left\| y - \hat{y} \right\| & = \left\| f(\mb{x}; \mb{w}, \mb{V}) - g(\mb{x}; \mb{\theta}) \right\| = \left\| a + \mb{b}^\top \mb{x} + \mb{x}^\top \mb{C} \mb{x} + \cdots  - R_{m+1}(\mb{x}) \right\|,
\end{align}\endgroup
which denotes the error introduced by using $m$-order polynomial equation to approximate the remainder term $R_{m+1}(\mb{x})$. By checking Linear Models, FM, MVM (and other machine learning models, like STM (support tensor machine) \cite{KGP12}), their main drawbacks lie in their lack of ability to model higher order feature interactions. It will lead great \textit{empirical error} in fitting such kinds of functions. In the following section, we will introduce that deep learning models can effectively resolve such a problem, which can fit any functions with any degree of accuracy universally.


\vspace{-10pt}
\section{Deep Learning Model Approximation Error Analysis}
\vspace{-10pt}

\begin{figure*}[!t]
\vspace{-30pt}
\begin{minipage}[b]{0.5\linewidth}
 \centering    
    \includegraphics[width=1.0\textwidth]{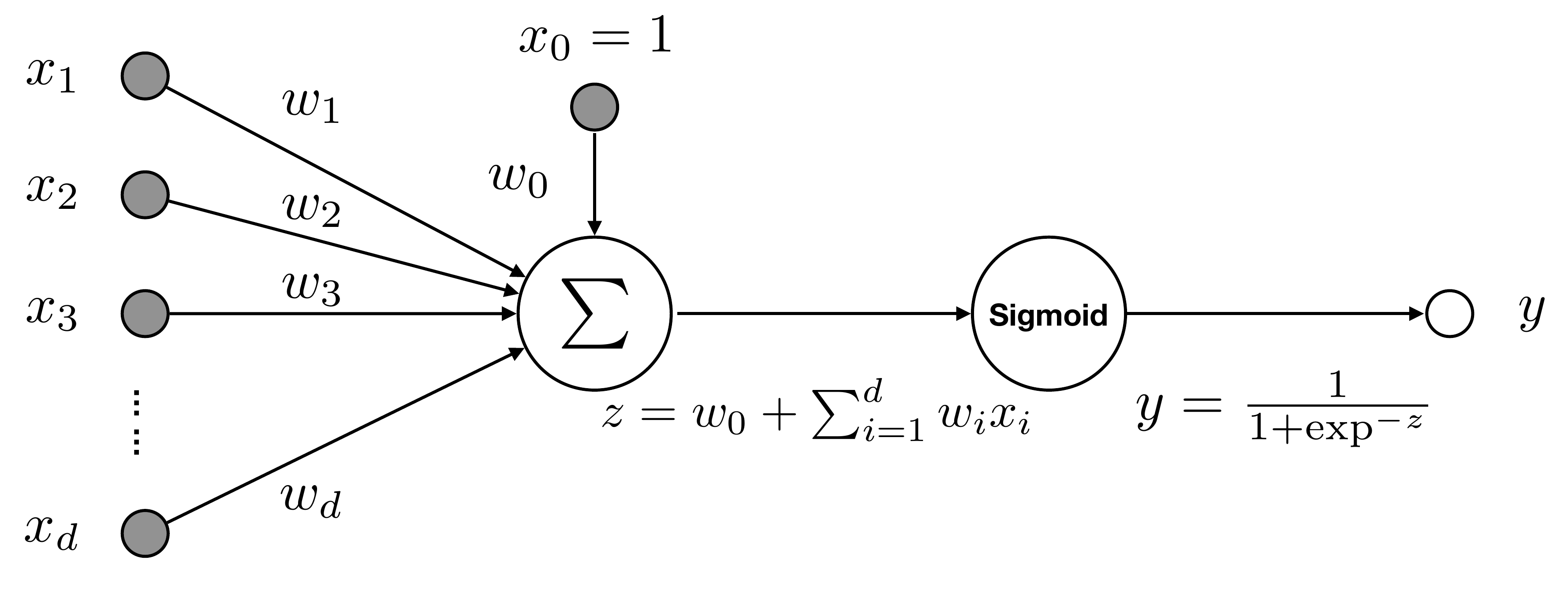}
    \vspace{-15pt}
\caption{Perceptron Model.}\label{fig:shallow}
\end{minipage}
\begin{minipage}[b]{0.5\linewidth}
  \centering
    \includegraphics[width=1.0\textwidth]{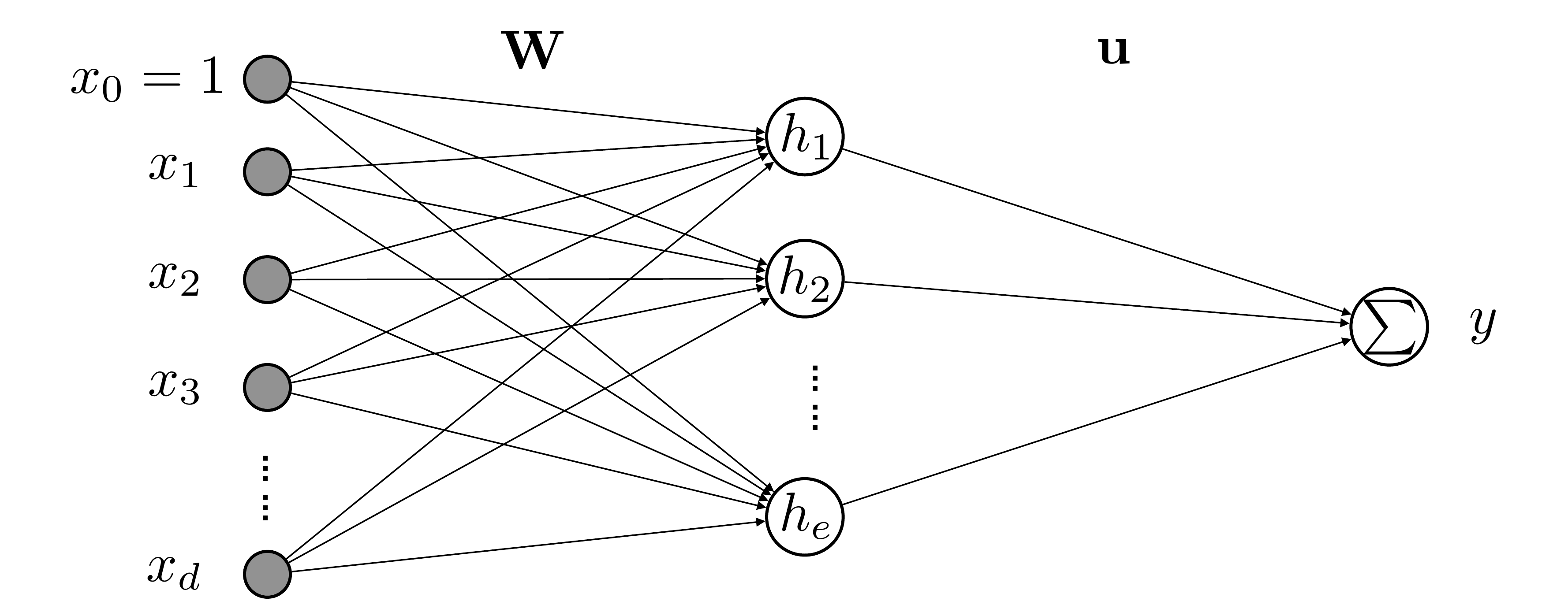}
    \vspace{-15pt}
\caption{Neural Network with Hidden Layer.}\label{fig:shallow_hidden}
 \end{minipage}
 \vspace{-20pt}
\end{figure*}

In this section, we will mainly focus on deep learning models approximation error analysis. At first, we will take perceptron neural network as an example and analyze its introduced approximation errors. After that, we will analyze the deep neural network models with hidden layers.


\vspace{-10pt}
\subsection{Perceptron Neural Network Approximation}
\vspace{-8pt}

As shown in Figure~\ref{fig:shallow}, we illustrate the architecture of a neural network model with a shallow architecture, merely involving the input layer and output layers respectively. The input for the model can be represented as vector $\mb{x} = [x_1, x_2, \cdots, x_{d_1}]^\top$, and terms $\{w_1, w_2, \cdots, w_{d_1}\}$ are the connection weights for each input feature, $w_0$ is the bias term. Here, we will use sigmoid function as the activation function. Formally, we can represent the prediction output for the input vector $\mb{x}$ as \vspace{-5pt} \begingroup\makeatletter\def\f@size{8}\check@mathfonts
\begin{equation}
\hat{y} = f(\mb{x}; \mb{w}) = \frac{1}{1+ \exp^{-(w_0 + \sum_{i=1}^{d_1} w_i \cdot x_i)} }
\end{equation}\endgroup

\vspace{-10pt} 
\noindent Sigmoid function has a good property, which can be expressed with the following lemmas.
\begin{lemma}\label{lemma:derivative}
Let $h\left( y^k (1-y)^l \right) = h_1\left( y^k (1-y)^l \right) + h_2\left( y^k (1-y)^l \right)$, where $h_1\left( y^k (1-y)^l \right) = ky^k(1-y)^{l+1}$ and $h_2\left( y^k (1-y)^l \right) = - ly^{k+1}(1-y)^l$. Given the $(n-1)_{th}$ derivative of $f(\mb{x}; \mb{w})$, i.e., $f^{(n-1)}(\mb{x}; \mb{w})$, the derivative of $f^{(n-1)}(\mb{x}; \mb{w})$ regarding $x_i$ can be represented as \vspace{-5pt} \begingroup\makeatletter\def\f@size{8}\check@mathfonts
\begin{equation}
f^{(n)}(\mb{x}; \mb{w}) = \frac{\partial f^{(n-1)}(\mb{x}; \mb{w})}{\partial x_i} = w_i h\left( f^{(n-1)}(\mb{x}; \mb{w}) \right).
\end{equation}\endgroup
\end{lemma}

\vspace{-12pt} 

In addition, to simplify the notations, we can use $h^{(n)}(\cdot) = h( h^{(n-1)}(\cdot) )$ to denote the $n_{th}$ recursive application of function $h(\cdot)$, where $h^{(0)}(y^k(1-y)^l) = y^k(1-y)^l$. Therefore, we can have the concrete representation for the $n_{th}$ derivative of function $f(\mb{x}; \mb{w})$ as follows:

\vspace{-5pt} 

\begin{lemma}\label{lemma:derivative2}
The $n_{th}$ derivative of function $f(\mb{x}; \mb{w})$ regarding the variables in $\mathcal{D} = [x_{i_1}, x_{i_1}, \cdots, x_{i_n}]$ (where $|\mathcal{D}| = n$ and $\mathcal{D} \neq \emptyset$) in a sequence can be represented as\begingroup\makeatletter\def\f@size{8}\check@mathfonts
\begin{equation}
\frac{\partial^n f(\mb{x}; \mb{w})}{\partial_{x_i \in \mathcal{D} } x_i} = \prod_{x_i \in \mathcal{D}} w_i \cdot h^{(n-1)}\left( f(\mb{x}; \mb{w}) \cdot (1- f(\mb{x}; \mb{w}))\right).
\end{equation}\endgroup
\end{lemma}

\vspace{-5pt} 

Based on the Lemmas, we can rewrite the approximation function for term $f(\mb{x}; \mb{w})$ at $\mb{x} = \mb{0}$ as the sum of a infinite polynomial equation according to the following theorem.

\vspace{-5pt} 

\begin{theo}\label{theo:expansion}
For the neural network model with sigmoid function as the activation function, its mapping function $f(\mb{x}; \mb{w})$ can be expanded as an infinite polynomial sequence at point $\mb{x}_0 = \mb{0}$: \begingroup\makeatletter\def\f@size{8}\check@mathfonts
\begin{equation}\label{equ:perceptron}
f(\mb{x}; \mb{w}) = v^{(0)} + \sum_{n = 1}^{d_1}  \sum_{1 \le i_1 < i_2 < \cdots < i_n \le d_1} v^{(k)}_{i_1,i_2, \cdots, i_n} \cdot x_{i_1} x_{i_2} \cdots x_{i_n},
\end{equation}\endgroup
where $v^{(0)} = f(\mb{0}; \mb{w})$ and $v^{(n)}_{i_1, \cdots, i_{n}} = w_{i_1} \big( \sum_{p=1}^{n-2} v^{(p)}_{i_1, \cdots, i_p} \cdot v^{(n-1-p)}_{i_{p+1}, \cdots, iå_{n-1}} + (1-2v^{(0)}) \cdot v^{(n-1)}_{i_{2}, \cdots, i_{n}} \big)$.
\end{theo}

\vspace{-5pt} 

\begin{proof}
According to Theorem~\ref{theo:representation}, the expansion of $f(\mb{x}; \mb{w})$ at $\mb{x} = \mb{0}$ can be represented as \vspace{-8pt} \begingroup\makeatletter\def\f@size{8}\check@mathfonts
\begin{align}
f(\mb{x}; \mb{w}) =v^{(0)} + \sum_{n = 1}^{d_1} \sum_{1 \le i_1 < i_2 < \cdots < i_n \le d_1} v^{(n)}_{{i_1,i_2, \cdots, i_n}} \cdot  x_{i_1} x_{i_2} \cdots x_{i_n}.
\end{align}\endgroup
Next, we will mainly focus on studying the relationship between the weight variables $\{v^{(0)}, v^{(1)}_{i_1}, \cdots v^{(d_1)}_{i_1, i_2, \cdots, i_{d_1}}\}$. Based on the mapping function, we have the concrete representation of the bias scalar terms of the polynomial terms in the expanded equation at expansion point $\mb{x}_0 = \mb{0}$ as $v^{(0)} = f(\mb{0}; \mb{w})$. Meanwhile, for the remaining terms, we propose to compute the derivative of $f(\mb{x}; \mb{w})$ regarding $x_{i_1}$ on both sides: \vspace{-3pt} \begingroup\makeatletter\def\f@size{8}\check@mathfonts
\begin{align}
&\frac{\partial f(\mb{x}; \mb{w})}{\partial x_{i_1}} =  \sum_{n = 1}^{d_1} \sum_{1 \le i_1 < i_2 < \cdots < i_n \le d_1} v^{(n)}_{{i_1,i_2, \cdots, i_n}} \cdot  x_{i_2} \cdots x_{i_n}  \\
&=   v_{i_1}^{(1)} + \sum_{i_2=1}^{d_1} v_{i_1,i_2}^{(2)} x_{i_2} + \sum_{i_2=1}^{d_1} \sum_{i_3=1}^{d_1} v_{i_1, i_2,i_3}^{(3)} x_{i_2} x_{i_3} + \cdots + \sum_{i_2=1}^{d_1} \cdots \sum_{i_{d_1}=1}^{d_1} v_{i_1, i_2, \cdots, i_{d_1}}^{(3)} x_{i_2} \cdots x_{i_{d_1}}.
\end{align}\endgroup
Here, we know that $\frac{\partial f(\mb{x}; \mb{w})}{\partial x_{i_1}} = w_{i_1} h( f(\mb{x}; \mb{w}) )$, therefore we can have $\forall \mb{x} \in \mathbb{R}^{d_1}$,\vspace{-6pt} \begingroup\makeatletter\def\f@size{8}\check@mathfonts
\begin{align}
&v_{i_1}^{(1)} + \sum_{i_2=1}^{d_1} v_{i_1,i_2}^{(2)} x_{i_2} + \sum_{i_2=1}^{d_1} \sum_{i_3=1}^{d_1} v_{i_1, i_2,i_3}^{(3)} x_{i_2} x_{i_3} + \cdots + \sum_{i_2=1}^{d_1} \cdots \sum_{i_{d_1}=1}^{d_1} v_{i_1, i_2, \cdots, i_{d_1}}^{(3)} x_{i_2} \cdots x_{i_{d_1}} \\
&= w_{i_1}  \big( v^{(0)} +  \sum_{ i_1 = 1}^{d_1}  v^{(1)}_{i_1} x_{i_1} + \sum_{i_1 = 1}^{d_1} \sum_{i_2 = 1}^{d_1} v^{(2)}_{i_1 , i_2} \cdot x_{i_1} x_{i_2} + \cdots + \sum_{i_1 = 1}^{d_1} \cdots \sum_{i_{d_1} = 1}^{d_1} v^{(d_1)}_{i_1, \cdots, i_{d_1}} \cdot x_{i_1} \cdots x_{i_{d_1}} \big) \\
& \big( 1- v^{(0)} -  \sum_{ i_1 = 1}^{d_1}  v^{(1)}_{i_1} x_{i_1} - \sum_{i_1 = 1}^{d_1} \sum_{i_2 = 1}^{d_1} v^{(2)}_{i_1 , i_2} \cdot x_{i_1} x_{i_2} - \cdots - \sum_{i_1 = 1}^{d_1} \cdots \sum_{i_{d_1} = 1}^{d_1} v^{(d_1)}_{i_1, \cdots, i_{d_1}} \cdot x_{i_1} \cdots x_{i_{d_1}}  \big).
\end{align}\endgroup
Therefore, we can have the $n_{th}$-order scalar weight for variable term $x_{i_1} \cdot ... \cdot x_{i_{n-1}}$ to be \vspace{-5pt} \begingroup\makeatletter\def\f@size{8}\check@mathfonts
\begin{align}
\begin{cases}
v^{(1)}_{i_1}& = w_{i_1} \cdot v^{(0)} (1 - v^{(0)}), \\
v^{(n)}_{i_1, \cdots, i_{n}} &= w_{i_1} \big( \sum_{p=1}^{n-2} v^{(p)}_{i_1, \cdots, i_p} \cdot v^{(n-1-p)}_{i_{p+1}, \cdots, iå_{n-1}} + (1-2v^{(0)}) \cdot v^{(n-1)}_{i_{2}, \cdots, i_{n}} \big).
\end{cases}
\end{align}\endgroup
\end{proof}

\vspace{-10pt}
\noindent In other words, for the perceptron neural network model, we have the $4$ key factors as follows: (1) $d_3 = d_1 + 1$, (2) $D' = 2^{d_1}$, (3) $\psi(\mb{w}) = [v^{(0)}] \oplus [v^{(1)}_{i}]_{i = 1}^{d_1} \oplus \cdots \oplus [v^{(d_1)}_{i_1,i_2, \cdots, i_{d_1}}]$, and (4) $\kappa(\mb{x}) = [1] \oplus [x_i]_{i = 1}^{d_1} \oplus \cdots \oplus [x_{i_1}x_{i_1} \cdots x_{i_{d_1}}]$. It seems the perceptron model can fit any high-order polynomial terms, since $\kappa(\mb{x})$ projects $\mb{x}$ to any high-order product of the features. However, according to an example problem to be shown in the next subsection, perceptron may fail to work well for non-monotone functions, e.g., XOR. Therefore, perceptron may introduce a large empirical loss in fitting non-monotone functions.

\vspace{-10pt}
\subsection{Deep Neural Network with Hidden Layers Approximation}
\vspace{-8pt}

We start this subsection with a deep neural network model with one hidden layer, as shown in Figure~\ref{fig:shallow_hidden}. In the plot, $\mb{x} = [1] \oplus [x_1, x_2, \cdots, x_{d_1}]^\top$ is the input feature vector, $\mb{W} \in \mathbb{R}^{(d_1 + 1) \times e}$ and $\mb{u} \in \mathbb{R}^{e \times 1}$ are the connection weight variables. Formally, we can represent the output neuron state as follows:\vspace{-5pt}\begingroup\makeatletter\def\f@size{8}\check@mathfonts
\begin{equation}
y = f(\mb{x}; \mb{w}) = \sum_{i = 1}^e u_i \cdot \sigma \left( \mb{w}_i^\top  \mb{x} \right),
\end{equation}\endgroup
where $\mb{w}_i = \mb{W}[:,i], \forall i \in \{1, 2, \cdots, e\}$. Although the neural network model shown in Figure~\ref{fig:shallow_hidden} has very simple structure, but it can be used as the foundation to learn the reconciled polynomial representation of any deep neural networks \cite{C89, M96, HSW89, H91}.

\begin{lemma}\label{lemma:universal}
Given any deep neural network model, denoted by function $f'(\mb{x}; \mb{w}')$, and any $\epsilon > 0$, function $f(\mb{x}; \mb{w}) = \sum_{i = 1}^e u_i \cdot \sigma \left( \mb{w}_i^\top  \mb{x} \right)$ can provide a good approximation of $f'(\mb{x}; \mb{w}')$ with some value $N$, i.e., \vspace{-5pt} \begingroup\makeatletter\def\f@size{8}\check@mathfonts
\begin{equation}
\left \| f'(\mb{x}; \mb{w}') - f(\mb{x}; \mb{w}) \right \| < \epsilon
\end{equation}\endgroup
\end{lemma}

\begin{theo}
Given any deep neural network model, denoted by function $f(\mb{x}; \mb{w})$, it can be approximately represented with the following polynomial summation \vspace{-5pt} \begingroup\makeatletter\def\f@size{8}\check@mathfonts
\begin{equation}
f(\mb{x}; \mb{w}) = v^{(0)} + \sum_{n = 1}^{d_1}  \sum_{1 \le i_1 < i_2 < \cdots < i_n \le d_1} v^{(k)}_{i_1,i_2, \cdots, i_n} \cdot x_{i_1} x_{i_2} \cdots x_{i_n},
\end{equation}\endgroup
where $v^{(0)} = \sum_{i = 1}^e u_i \cdot  v^{i,(0)}$ and $v^{(n)}_{i_1, \cdots, i_{n}} = \sum_{i = 1}^e u_i \cdot v^{i,(k)}_{i_1,i_2, \cdots, i_n} $, and $v^{i,(0)}$ and $v^{i,(k)}_{i_1,i_2, \cdots, i_n}$ can be represented according to Theorem~\ref{theo:expansion}.
\end{theo}

\begin{proof}
According to Theorem~\ref{theo:expansion}, the sigmoid function $\sigma \left( \mb{w}_i^\top  \mb{x} \right)$ can be represented as a reconciled polynomial summation. Therefore, we have equation\begingroup\makeatletter\def\f@size{8}\check@mathfonts
\begin{align}
f(\mb{x}; \mb{w}) &= \sum_{i = 1}^e u_i \cdot \sigma \left( \mb{w}_i^\top  \mb{x} \right)= \sum_{i = 1}^e u_i \cdot  \left( v^{i,(0)} + \sum_{n = 1}^{d_1}  \sum_{1 \le i_1 < i_2 < \cdots < i_n \le d_1} v^{i,(k)}_{i_1,i_2, \cdots, i_n} \cdot x_{i_1} x_{i_2} \cdots x_{i_n} \right),\\
& \begin{matrix} =  \underbrace{\sum_{i = 1}^e u_i \cdot  v^{i,(0)} }  &+ \sum_{n = 1}^{d_1}  \sum_{1 \le i_1 <  \cdots < i_n \le d_1} & \underbrace{ \left(\sum_{i = 1}^e u_i \cdot v^{i,(k)}_{i_1,i_2, \cdots, i_n}  \right) }   & x_{i_1}  \cdots x_{i_n},\\
\ \ \ \ v^{(0)}  & &v^{(k)}_{i_1,i_2, \cdots, i_n}  & 
 \end{matrix}
\end{align}\endgroup
\end{proof}

\vspace{-10pt}
Furthermore, according to Lemma~\ref{lemma:universal}, we can conclude that any deep neural network model can also be unified represented as the reconciled polynomial summation equation proposed in this paper, where the key $4$ factors are: (1) $d_3 = (d_1 + 2) \times e$, (2) $D' = 2^{d_1}$, (3) $\psi(\mb{w}) = [v^{(0)}] \oplus [v^{(1)}_{i}]_{i = 1}^{d_1} \oplus \cdots \oplus [v^{(d_1)}_{i_1,i_2, \cdots, i_{d_1}}]$, and (4) $\kappa(\mb{x}) = [1] \oplus [x_i]_{i = 1}^{d_1} \oplus \cdots \oplus [x_{i_1}x_{i_1} \cdots x_{i_{d_1}}]$. According to Lemma~\ref{lemma:universal}, the deep neural network models can achieve almost $0$ empirical loss in fitting any functions defined based on the feature and label space. Meanwhile, since there exist a large number of variables to be learned in deep neural network models, learning a model which can achieve $0$ expected loss so as to minimize the gap between the overall loss vs. empirical loss in a challenging task. More analysis about the deep neural network models will be provided in the following subsection.

\vspace{-10pt}
\subsection{Analysis of Deep Neural Network Model Advantages}
\vspace{-8pt}

Deep neural networks \cite{H91, KSH12, BLPL06, VLLBM10, J02, HOT06} have much more powerful function representation capacity than the perceptron neural network model, e.g., in fitting the XOR dataset and some other more complicated datasets. Considering that all the deep neural network models can be effectively approximated with the multiple layer perceptron model introduced in the previous section, we may wonder are there any other advantages of deeper neural networks compared with shallower neural networks \cite{MLP17}. In this part, we will mainly focus on analyze the advantages of deep neural networks.

\begin{theo}\label{theo:deep_vs_shallow}
For any desired accuracy $\epsilon > 0$, for any function $g(\mb{x}; \mb{\theta})$, it can be approximated with either (1) a deep neural network with $K$ hidden layers ($K \ge 2$) and $M$ variables, or (2) a shallow neural network model with $1$ hidden and $N$ variables, where $M \ll N$.
\end{theo}


To demonstrate the above Theorem, we need to introduce the following two Lemmas.

\begin{lemma}\label{theo:shallow}
For any desired accuracy $\epsilon > 0$, there exist a neural network with a single layer that can implement a polynomial function $\prod_{i = 1}^n x_i$, where the minimum number of involved hidden neurons is $2^n$.
\end{lemma}

\begin{lemma}\label{theo:deep2}
For any desired accuracy $\epsilon > 0$, the polynomial function $\prod_{i = 1}^n x_i$ can be implemented with a $k$-layer deep neural network model with $2^{n^{\frac{1}{k}}} \cdot \frac{ 1-n }{ 1-n^{\frac{1}{k}} }$ hidden neurons, which takes equal sized input from lower level.
\end{lemma}

\begin{proof}
Based on the above two Lemmas, we provide the proof of Theorem~\ref{theo:deep_vs_shallow} as follows. According to Theorem~\ref{theo:representation}, any function defined on space $\mathcal{X}$ can be represented as a finite sum of polynomial terms, where the maximum order of the terms is $\prod_{i = 1}^{d_1} x_i = (x_1 \cdot (x_2 \cdots ( ... (x_{d_1-1} \cdot x_{d_1}) ) )$. (For the other terms with order less than $d_1$, e.g., $x_1 \cdot x_2$, we can also be represented as $\prod_{i = 1}^{d_1} x_i$, where $x_3=x_4 \cdots = x_{d_1} = 1$).

According to Theorems~\ref{theo:shallow}-\ref{theo:deep}, we know the 1-hidden layer neural network may need $2^n$ hidden neurons and the deep neural network taking pairwise inputs merely need $n-1$ hidden neurons, i.e., $N = 2^n$ and $M = n-1$. We can show that $N$ is almost the exponential  as $M$.
\end{proof}

\begin{theo}\label{theo:deep}
For any desired accuracy $\epsilon > 0$, the polynomial function $\prod_{i = 1}^n x_i$ can be implemented with a deep neural network model with a minimum $2^2 \cdot (n-1)$ hidden neurons. Meanwhile, the number of hidden layers involved the deep neural network will range from $n - 1$ to $\lceil \log n \rceil + 1$.
\end{theo}

\begin{proof}
We propose to prove the above theorem by constructing a deep neural network model based on Theorem~\ref{theo:shallow}. Formally, term $\prod_{i = 1}^n x_i$ can be precisely represented as an recursive product of two features, e.g., $\prod_{i = 1}^n x_i = (x_1 \cdot (x_2 \cdots ( ... (x_{n-1} \cdot x_n) ) )$. Depending on how we construct the product function term, the computation process can be represented with binary trees in different shapes, where $\{x_1, x_2, \cdots, x_n\}$ as the leaf nodes and the \textit{intermediate result} as the internal nodes.

For the binary tree with $n$ leaf nodes, the number of internal nodes involved in it will be $n-1$, regardless of the tree shape. For each internal node in the binary tree, it will accept two inputs from either leaf node or internal node. Meanwhile, according to Theorem~\ref{theo:shallow}, the computation of each internal node can be approximated to any accuracy with a 1-hidden layer neural network model involving $2^2$ hidden neurons (i.e., $(2+1) \times 2^2$ variables).

The specific number of layers involved is highly dependent on the shape of binary tree about the neural network model. For instance, for the function decomposed in the shape of a full binary tree (just like function $\prod_{i = 1}^n x_i = (x_1 \cdot (x_2 \cdots ( ... (x_{n-1} \cdot x_n) ) )$), its level (i.e., neural network hidden layer) will be $n-1$, which is the maximum hight. Meanwhile, for the function decomposed in the shape of a complete binary tree, its number of levels (i.e., neural network hidden layer) will be $\lceil \log n \rceil + 1$, which is the minimum height.
\end{proof}

\begin{theo}
Given two deep neural network models $f_1(\mb{x}; \mb{w}_1)$ and $f_2(\mb{x}; \mb{w}_2)$ which contains the required number of hidden neurons than the required by Theorem~\ref{theo:deep2}. Let $| \mb{w}_1 | = n_1$ and $| \mb{w}_2 | = n_2 $ denote the number of variables involved in them respectively. If $n_1 < n_2$, then model $f_1(\mb{x}; \mb{w}_1)$ is more likely to achieve less approximation error than $f_2(\mb{x}; \mb{w}_2)$.
\end{theo}

\begin{proof}
According to Section~\ref{sec:representation}, we have shown that the function approximation loss can be categorized into the \textit{empirical loss} and \textit{expected loss}. To learn a good approximation function to the unknown true mapping function, it involves two objectives: (1) minimizing the \textit{empirical loss}, and (2) minimizing the gap between \textit{empirical loss} and \textit{overall loss}. According to the previous Lemma~\ref{lemma:universal}, the deep neural network models can provide universal approximation to any functions. In other words, given a training dataset $\mathcal{T}$, we are able to build a deep neural network achieving nearly $0$ \textit{empirical loss}, if the required number hidden neurons are available in the model. Here, we can represent the maximum loss introduced by a deep neural network model on training set $\mathcal{T}$ as $e_{m} = \max\left( \{f(\mb{x}; \mb{w}) - g(\mb{x}; \mb{\theta})\}_{\mb{x} \in \mathcal{T}} \right)$. In other words, we have $f(\mb{x}; \mb{w}) - g(\mb{x}; \mb{\theta}) \in [0, e_m], \forall \mb{x} \in \mathcal{T}$.

Next, we will mainly focus on analyzing the difference between \textit{empirical loss} and \textit{overall loss}. According to the Hoeffding's Inequality, given the \textit{empirical loss} computed based on the available training set within range $[0, e_m]$, the probability that \textit{empirical loss} of neural network with a specific variable vector has a gap greater than $\epsilon$ compared with the \textit{overall loss} can be represented as:\vspace{-5pt} \begingroup\makeatletter\def\f@size{8}\check@mathfonts
\begin{equation}
P(\left| \mathcal{L}_{em} - \mathcal{L} \right| \ge \epsilon | \mb{w}) \le 2 \exp^{\left( - \frac{2 |\mathcal{T}|^2 \epsilon^2 }{\sum_{\mb{x} \in \mathcal{T}} e_m^2} \right)} = 2 \exp^{\left( - \frac{2 |\mathcal{T}| \epsilon^2 }{e_m^2} \right)}.
\end{equation}\endgroup

\vspace{-10pt}
According to the inequality, we can draw many conclusions: (1) larger training set are helpful, and (2) stronger models are helpful. For a training set of a larger size, i.e., $|\mathcal{T}|$ is larger, we have probability $P(\left| \mathcal{L}_{em} - \mathcal{L} \right| \ge \epsilon)$ will be smaller. Meanwhile, for a stronger approximation model, i.e., achieving an extremely small \textit{empirical loss} $e_m$ on the training set $\mathcal{T}$, probability $P(\left| \mathcal{L}_{em} - \mathcal{L} \right| \ge \epsilon)$ will be of a smaller value as well.

Meanwhile, the above inequality are for neural network model a specific variable. However, in the application, for each neural network model with different numbers of variables, it will introduce many feasible variable solutions. Formally, for the variables of neural network $f(\mb{x}; \mb{w})$, i.e., $\mb{w}$, we can represent its feasible values for the given training set $\mathcal{T}$ as $\mathcal{W}(\mathcal{T})$. By considering all the feasible model variables in $\mathcal{W}(\mathcal{T})$, we can rewrite the inequality as follows:\vspace{-5pt}\begingroup\makeatletter\def\f@size{8}\check@mathfonts
\begin{equation}
P\left(\left| \mathcal{L}_{em} - \mathcal{L} \right| \ge \epsilon | \mathcal{W}(\mathcal{T}) \right) \le \sum_{\mb{w} \in \mathcal{W}( \mathcal{T})} P\left(\left| \mathcal{L}_{em} - \mathcal{L} \right| \ge \epsilon | \mb{w} \right) \le 2 \cdot | \mathcal{W}( \mathcal{T}) | \cdot \exp^{\left( - \frac{2 |\mathcal{T}| \epsilon^2 }{e_m^2} \right)}.
\end{equation}\endgroup

\vspace{-10pt}
The size of feasible variable set $\mathcal{W}(\mathcal{T})$ with given training set $\mathcal{T}$ usually increases exponentially with the size of variables involved in the neural network model, which also proves this theorem.
\end{proof}

\vspace{-10pt}
\section{Appendix}
\vspace{-10pt}

\noindent \textbf{Proof of Lemma~\ref{theo:shallow}}: According to \cite{LT16}, $2^n$ hidden neurons will be sufficient and necessary to model the polynomial product term $\prod_{i = 1}^n x_i$, and the prove will not be provided here. In such a case, the number of involved variables will be $(n+1)\cdot 2^n$.

\noindent \textbf{Proof of Lemma~\ref{theo:deep2}}: The proof introduced here is innovated by \cite{RT17}, which demonstrate the above theorem by constructing a deep neural network based on Theorem~\ref{theo:shallow}. Given the function term $\prod_{i = 1}^n x_i$, we propose to partition the features $\{x_1, x_2, \cdots, x_n\}$ into several groups layer by layer. For instance, at the $1_{st}$, we can partition $\{x_1, x_2, \cdots, x_n\}$ into several groups of size $b_1$, where the number of groups will be $\frac{n}{b_1}$. With in each group, the product of the features can be effectively approximated with $2^{b_1}$ hidden neurons, and the total number hidden neurons at the first layer will be $\frac{n}{b_1} \cdot 2^{b_1}$.

At the $2_{nd}$ layer, we further partition the features with in each groups into several groups of size $b_2$, and the total number of subgroups at layer $2$ will be $\frac{n}{b_1 \cdot b_2}$. Such a process continues until the $k_{th}$ layer, where $n = b_1 \cdot b_2 \cdots b_k$. Therefore, the total number of hidden neurons involved in the constructed deep neural network will be $m \le \sum_{i = 1}^k \frac{n}{ \prod_{j = 1}^i b_j} 2^{b_i} = \sum_{i = 1}^k \left( \prod_{j = i+1}^k b_j \right) 2^{b_i}$.
By setting $b_i = n^{\frac{1}{k}}$, we have $m \le \sum_{i = 1}^k \left( \prod_{j = i+1}^k n^{\frac{1}{k}} \right) 2^{n^{\frac{1}{k}}} = 2^{n^{\frac{1}{k}}} \sum_{i = 1}^k n^{\frac{k - i}{k}} = 2^{n^{\frac{1}{k}}} \cdot \frac{ 1-n }{ 1-n^{\frac{1}{k}} }$.


\bibliographystyle{abbrv}
\bibliography{reference}

\begin{thebibliography}{10}

\bibitem{BLPL06}
Y.~Bengio, P.~Lamblin, D.~Popovici, and H.~Larochelle.
\newblock Greedy layer-wise training of deep networks.
\newblock In {\em NIPS}, 2006.

\bibitem{BEHW89}
A.~Blumer, A.~Ehrenfeucht, D.~Haussler, and M.~Warmuth.
\newblock Learnability and the vapnik-chervonenkis dimension.
\newblock {\em J. ACM}, 1989.

\bibitem{CZLY16}
B.~Cao, H.~Zhou, G.~Li, and P.~Yu.
\newblock Multi-view machines.
\newblock In {\em WSDM}, 2016.

\bibitem{CV95}
C.~Cortes and V.~Vapnik.
\newblock Support-vector networks.
\newblock {\em Mach. Learn.}, 1995.

\bibitem{C89}
G.~Cybenko.
\newblock {Approximation by superpositions of a sigmoidal function}.
\newblock {\em Mathematics of Control, Signals, and Systems (MCSS)}, 1989.

\bibitem{HOT06}
G.~Hinton, S.~Osindero, and Y.~Teh.
\newblock A fast learning algorithm for deep belief nets.
\newblock {\em Neural Comput.}, 2006.

\bibitem{H91}
K.~Hornik.
\newblock Approximation capabilities of multilayer feedforward networks.
\newblock {\em Neural Netw.}, 1991.

\bibitem{HSW89}
K.~Hornik, M.~Stinchcombe, and H.~White.
\newblock Multilayer feedforward networks are universal approximators.
\newblock {\em Neural Netw.}, 1989.

\bibitem{J02}
H.~Jaeger.
\newblock {Tutorial on training recurrent neural networks, covering BPPT, RTRL,
  EKF and the ``echo state network'' approach}.
\newblock Technical report, Fraunhofer Institute for Autonomous Intelligent
  Systems (AIS), 2002.

\bibitem{KGP12}
I.~Kotsia, W.~Guo, and I.~Patras.
\newblock Higher rank support tensor machines for visual recognition.
\newblock {\em Pattern Recogn.}, 2012.

\bibitem{KSH12}
A.~Krizhevsky, I.~Sutskever, and G.~Hinton.
\newblock Imagenet classification with deep convolutional neural networks.
\newblock In {\em NIPS}, 2012.

\bibitem{LT16}
H.~Lin and M.~Tegmark.
\newblock Why does deep and cheap learning work so well?, 2016.

\bibitem{M96}
H.~Mhaskar.
\newblock Neural networks for optimal approximation of smooth and analytic
  functions.
\newblock {\em Neural Comput.}, 1996.

\bibitem{MLP17}
H.~Mhaskar, Q.~Liao, and T.~Poggio.
\newblock When and why are deep networks better than shallow ones?
\newblock In {\em AAAI}, 2017.

\bibitem{R10}
S.~Rendle.
\newblock Factorization machines.
\newblock In {\em ICDM}, 2010.

\bibitem{RT17}
D.~Rolnick and M.~Tegmark.
\newblock The power of deeper networks for expressing natural functions.
\newblock {\em CoRR}, abs/1705.05502, 2017.

\bibitem{VC15}
V.~Vapnik and A.~Chervonenkis.
\newblock {\em On the Uniform Convergence of Relative Frequencies of Events to
  Their Probabilities}.
\newblock 2015.

\bibitem{VLLBM10}
P.~Vincent, H.~Larochelle, I.~Lajoie, Y.~Bengio, and P.~Manzagol.
\newblock Stacked denoising autoencoders: Learning useful representations in a
  deep network with a local denoising criterion.
\newblock {\em J. Mach. Learn. Res.}, 2010.

\bibitem{YS09}
X.~Yan and X.~Su.
\newblock {\em Linear Regression Analysis: Theory and Computing}.
\newblock 2009.

\end{thebibliography}

\end{document}